\documentclass{article}

\usepackage[accepted]{icml2026}

\usepackage{amsmath,amssymb,amsthm}
\usepackage{mathtools}
\usepackage{bbm}
\usepackage{bm}
\usepackage{hyperref}
\usepackage{cleveref}
\usepackage{booktabs}
\usepackage{microtype}
\usepackage{graphicx}
\usepackage{xcolor}
\usepackage{enumitem}
\usepackage[ruled,linesnumbered,algo2e]{algorithm2e}

\hypersetup{colorlinks=true,
  linkcolor=blue!70!black,citecolor=blue!50!black,urlcolor=blue!60!black}
  
\providecommand{\E}{\mathbb{E}}
\providecommand{\Prob}{\mathbb{P}}
\providecommand{\defeq}{\coloneqq}

\graphicspath{{figures/}{../figures/}{./}}
\Crefname{figure}{Fig.}{Figs.}
\crefname{figure}{Fig.}{Figs.}

\newtheorem{lemma}{Lemma}
\newtheorem{proposition}{Proposition}

\newtheorem{assumption}{Assumption}
\theoremstyle{definition}

\newtheorem{remark}{Remark}

\newcommand{\cN}{\mathcal{N}}

\newcommand{\cW}{\mathcal{W}}

\newcommand{\Dtr}{\mathcal{D}_{\mathrm{tr}}} 
\newcommand{\qs}{q_s}                        
\newcommand{\qt}{q_t}                        
\newcommand{\ps}{p_s}                        
\newcommand{\pt}{\widetilde{p}_{w}}   

\newcommand{\ws}{w}
\newcommand{\Zx}[1][x]{Z_w(#1)}
\newcommand{\NLPD}{\mathrm{NLPD}}
\newcommand{\KL}{D_{\mathrm{KL}}}
\newcommand{\ei}{e_i}

\newcommand{\Fil}{\mathcal{F}}
\newcommand{\iid}{\overset{\mathrm{iid}}{\sim}}

\icmltitlerunning{Anytime-Valid Confirmation of Label-Shift Corrections}

\begin{document}

\twocolumn[
\icmltitle{Anytime-Valid Confirmation of Label-Shift Corrections}

\begin{icmlauthorlist}
\icmlauthor{Seungjin Choi}{croid}
\end{icmlauthorlist}

\icmlaffiliation{croid}{CROID Research and aSSIST University, Seoul, Korea}
\icmlcorrespondingauthor{Seungjin Choi}{seungjin.choi.mlg@gmail.com}

\icmlkeywords{e-values, e-processes, label shift, sequential testing,
              anytime-valid inference, predictive confirmation, Gaussian processes}
\vskip 0.3in
]

\printAffiliationsAndNotice{}

\begin{abstract}
In small-batch scientific deployments, labeled target outcomes may be
too scarce for reliable shift estimation even when unlabeled target
inputs are available.  We address the complementary setting where
the practitioner has a pre-specified label-shift correction from domain
knowledge and asks whether incoming labeled outcomes support it.
We show that the per-observation likelihood ratio between a
label-shift-corrected predictive and the source predictive is a
conditional e-value, so its running product is a nonnegative martingale
and Ville's inequality yields an anytime-valid confirmation rule.  
The log martingale equals the cumulative negative log-predictive density (NLPD) gap 
between the source and the corrected predictive, converting routine model monitoring into a formal sequential test.
Rejection means the incoming data support the posited correction relative to the source predictive, 
but it is not a precise estimate of the degree of shift.
Closed forms are available for GP sources with Gaussian label-shift ratios.  
GP regression simulations validate Type I control, finite-sample power,
miscalibration sensitivity, and the small-batch advantage of a reliable
prior over label-based re-estimation.
\end{abstract}

\section{Introduction}
\label{sec:intro}
 
When deploying a predictive model in a new environment, the label
distribution may shift between source and target domains,
a phenomenon known as \emph{label shift}, where the conditional distribution
$p(x\,|\,y)$ is stable but the marginal $p(y)$ changes across domains.  
Label shift arises broadly, from clinical deployment across hospital sites to scientific instrument migration 
and manufacturing process changes, wherever the frequency of outcomes differs between training and test conditions.
 
Classical label-shift adaptation is usually framed as an estimation
problem.  In identifiable settings, most notably classification with a
well-conditioned source predictor, one can use sufficient \emph{unlabeled}
target inputs to estimate the shift ratio
$\ws(y)=p_t(y)/p_s(y)$ without target labels and then retilt the source
predictive \citep{LiptonZ2018icml,AlexandariAM2020icml,GargS2020neurips}.
This is the right tool when unlabeled target data are available in sufficient
quantity and the estimation problem is well conditioned.
 
The setting considered in this paper is different.  In many deployments,
\emph{labeled} target outcomes arrive sequentially and may be scarce,
unlabeled target inputs alone may be insufficient to estimate the
shift reliably, and practitioners already have an externally specified
candidate correction, such as a regulatory specification, a bridging
experiment, or domain knowledge about the direction and approximate
degree of the shift.  The question then changes from
\emph{``what is the shift?''} to \emph{``is the proposed correction
supported by the incoming labeled outcomes?''}  This is a confirmation
problem, not a shift-estimation problem, and it calls for a different
tool.
 
This paper provides that tool.  Given a source predictive
$\ps(y\,|\,x,\Dtr)$ and a candidate weight $\ws(y)$ derived from domain
knowledge, we form the tilted predictive
$\pt(y\,|\,x,\Dtr)\propto\ws(y)\,\ps(y\,|\,x,\Dtr)$ and accumulate the
per-observation likelihood ratio between $\pt$ and $\ps$.  We show this
ratio is a conditional e-value under the predictive null, and its running
product is a nonnegative martingale, also known as an e-process
\citep{VovkV2021aos,ShaferG2021jrsssa,RamdasA2023ss}.  Ville's inequality
then gives an anytime-valid confirmation rule in which Type I error is
controlled at every sample size simultaneously, with no correction for
optional stopping or repeated looks at the data.
 
One point deserves emphasis.  Rejection at time $\tau^*$ means the
cumulative negative log-predictive density (NLPD) gap has exceeded $\log(1/\alpha)$, which by Ville's
inequality has probability at most $\alpha$ under $H_0^{\rm pred}$.
This confirms the posited correction is consistent with the incoming
data, but does not quantify the degree of shift.
For quantitative shift estimation, unlabeled-target methods such as
BBSE \citep{LiptonZ2018icml} remain the appropriate tool.
 
\if(0)
The key structural result (\eqref{eq:nlpdgap}) makes this practically
transparent because $\log M_t$ equals the cumulative NLPD improvement of $\pt$
over $\ps$, a quantity practitioners already compute for model monitoring.
The only additional step is comparing it to $\log(1/\alpha)$.
\fi
 
This paper contributes (i)~the NLPD-gap identity connecting
Kelly-optimal betting \citep{KellyJL56bell} to model evaluation;
(ii)~an anytime-valid confirmation test with no pre-specified sample
size; (iii)~closed-form GP e-values under Gaussian label-shift tilts;
(iv)~a growth-rate characterization giving asymptotic power one under
correct specification; and (v)~empirical demonstrations of Type I
control, miscalibration sensitivity, small-batch advantage over noisy
label-based re-estimation, and robustness to moderate shift
misspecification.
 
The theory of e-values and e-processes is developed in
\citep{VovkV2021aos,ShaferG2021jrsssa,RamdasA2023ss}.
Growth-rate-optimal (GRO) e-values are characterized by \citet{GrunwaldP2024jrsssb}.
Our construction deliberately uses a fixed domain-knowledge-based weight rather than a GRO weight,
which would require specifying the alternative more precisely.  
The same predictive-level tilting construction was used in our companion
work on conformal Bayes under label shift \citep{Choi2026eiml}, where the tilted
predictive determines the nonconformity score geometry and is paired with an
importance-weighted conformal quantile for target-domain coverage.  

In this paper, the same tilted/source predictive pair serves a different purpose.
The normalized ratio itself becomes a sequential e-value, yielding anytime-valid
confirmation of a fixed posited shift rather than conformal prediction sets.
Our test differs from change-point detectors
\citep{VovkV2005book,XuC2023icml} and weighted conformal prediction
\citep{TibshiraniR2019neurips,PodkopaevA2021uai}, which test an unknown
alternative or provide coverage guarantees rather than anytime-valid
confirmation of a fixed posited shift.

\section{Background and Notation}
\label{sec:background}

\subsection{E-Values, E-Processes, and Ville's Inequality}

Let $\Fil_t=\sigma(\Dtr,X_1,Y_1,\ldots,X_t,Y_t)$ denote the natural
filtration (all information up to time $t$).  A nonnegative random
variable $e$ is an \emph{e-value} for $H_0$ if $\E_{H_0}[e]\le1$.
E-values compose sequentially.  If $\E_{H_0}[e_i\,|\,\Fil_{i-1}]\le1$
for each $i$, then $M_t=\prod_{i=1}^t e_i$ is a nonnegative
supermartingale (an \emph{e-process}) under $H_0$
\citep{RamdasA2023ss}; when equality holds it is a martingale.
Ville's inequality gives
\[
  \Prob_{H_0}\!\left(\sup_{t\ge0}M_t>\tfrac{1}{\alpha}\right)\le\alpha,
\]
so the stopping rule $\tau^*=\inf\{t:M_t>1/\alpha\}$ controls Type I
error at level $\alpha$ at every sample size simultaneously.  This is the
\emph{anytime-valid} guarantee.  The simplest e-value for testing
null $Y_i\,|\,X_i\sim p_0(\cdot\,|\,x_i)$ against alternative
$Y_i\,|\,X_i\sim p_1(\cdot\,|\,x_i)$ is the likelihood ratio
$e_i=p_1(Y_i\,|\,X_i)/p_0(Y_i\,|\,X_i)$, since
$\E_{p_0}[e_i\,|\,X_i,\Fil_{i-1}]=1$.

\subsection{Predictive Null, Tilted Predictive, and NLPD}

We test the \emph{predictive null}
\begin{equation}
  H_0^{\rm pred}:\;Y_i\,|\,X_i,\Fil_{i-1}\sim\ps(\cdot\,|\,X_i,\Dtr),
  \label{eq:prednull}
\end{equation}
where $\ps(y\,|\,x,\Dtr)$ is the source posterior predictive, trained on
$\Dtr$ and held fixed.  This differs from the data-level no-shift null
$H_0^{\rm data}:\qt=\qs$; the two coincide only when $\ps$ is
well-calibrated for $\qs$.

Under label shift $\qs(x\,|\,y)=\qt(x\,|\,y)$, so the density ratio
$\qt/\qs=\qt(y)/\qs(y)\eqqcolon\ws_{\rm true}(y)$ depends only on $y$.
Given a weight $\ws(y)$ from domain knowledge, the \emph{tilted predictive} is given by
\begin{equation}
  \pt(y\,|\,x,\Dtr)\defeq\frac{\ws(y)\,\ps(y\,|\,x,\Dtr)}{\Zx},
  \label{eq:tilted}
\end{equation}
where $\Zx=\int\ws(y)\,\ps(y\,|\,x,\Dtr)\,dy$.
The per-observation e-value is then
\[
  e_i=\pt(y_i\,|\,x_i,\Dtr)/\ps(y_i\,|\,x_i,\Dtr)=\ws(y_i)/\Zx[x_i].
\]
The \emph{negative log-predictive density} (NLPD) of predictive $p$ at
$(x_i,y_i)$ is $\NLPD_i(p)\defeq-\log p(y_i\,|\,x_i,\Dtr)$.  With
$\NLPD_i^s\defeq\NLPD_i(\ps)$ and $\NLPD_i^w\defeq\NLPD_i(\pt)$
(superscript $w$ for posited weight, not time), $\log e_i=\NLPD_i^s-\NLPD_i^w$,
so the log-wealth is the cumulative NLPD gap
\begin{equation}
  \log M_t=\sum_{i=1}^t(\NLPD_i^s-\NLPD_i^w).
  \label{eq:nlpdgap}
\end{equation}

\section{From Predictive Tilting to Anytime-Valid E-Processes}
\label{sec:main}

\subsection{E-Value and E-Process Identification}
\label{sec:evalue}

Throughout this subsection the candidate weight $\ws$ is fixed before the
stream begins, $\Zx[x]=\int \ws(y)\ps(y\,|\,x,\Dtr)dy$ is finite and
positive, and
\[
  \pt(y\,|\,x,\Dtr)=\frac{\ws(y)\ps(y\,|\,x,\Dtr)}{\Zx[x]}.
\]
The inputs $X_i$ may be stochastic, deterministic, or adaptively selected;
what matters for validity is the conditional distribution of $Y_i$ under the
predictive null.

\begin{lemma}[Tilted predictive as likelihood ratio]
  \label{lem:lr}
  For every $x$ with $0<\Zx[x]<\infty$,
  \[
    \frac{\pt(y\,|\,x,\Dtr)}{\ps(y\,|\,x,\Dtr)}=\frac{\ws(y)}{\Zx[x]},
    \qquad
    \E_{Y\sim\ps(\cdot\,|\,x,\Dtr)}
    \left[\frac{\ws(Y)}{\Zx[x]}\right]=1.
  \]
\end{lemma}

\emph{Proof.} See \cref{app:proof-lem-lr}.

\begin{proposition}[Per-sample predictive e-value]
  \label{prop:evalue}
  Define
  \[
    e_i\defeq
    \frac{\pt(Y_i\,|\,X_i,\Dtr)}{\ps(Y_i\,|\,X_i,\Dtr)}
    =\frac{\ws(Y_i)}{\Zx[X_i]}.
  \]
  Under $H_0^{\rm pred}$ in \eqref{eq:prednull},
  \[
    \E[e_i\,|\, \Fil_{i-1},X_i]=1.
  \]
  Hence $e_i$ is an e-value conditionally on the realized input $X_i$.
\end{proposition}

\emph{Proof.} See \cref{app:proof-prop-evalue}.

\begin{proposition}[Prequential predictive e-process]
  \label{prop:eprocess}
  Assume $H_0^{\rm pred}$ and define $M_0=1$,
  \[
    M_t\defeq\prod_{i=1}^{t}e_i.
  \]
  Then $(M_t)_{t\ge0}$ is a nonnegative martingale with respect to
  $(\Fil_t)_{t\ge0}$.  In particular, it is an e-process under the
  predictive null.
\end{proposition}

\emph{Proof.} See \cref{app:proof-thm-eprocess}.

By \eqref{eq:nlpdgap}, $\log M_t=\sum_{i=1}^t(\NLPD_i^s-\NLPD_i^w)$
for every realized stream.  Crossing $\log(1/\alpha)$ is therefore
exactly the event that the tilted predictive has accumulated more than
$\log(1/\alpha)$ nats of log-score advantage over the source.  The null
assumption enters only when converting this algebraic identity into the
anytime-valid test below.

\begin{proposition}[Anytime-valid predictive confirmation test]
  \label{prop:anytime}
  Under $H_0^{\rm pred}$, for any $\alpha\in(0,1)$ the stopping rule
  \[
    \tau^*\defeq\inf\{t\ge1:M_t>1/\alpha\}
  \]
  satisfies
  \[
    \Prob_{H_0^{\rm pred}}(\tau^*<\infty)\le\alpha.
  \]
\end{proposition}

\emph{Proof.} See \cref{app:proof-thm-anytime}.

\begin{remark}[What rejection means]
  \label{rem:rejection_meaning}
  Rejecting at time $\tau^*$ means the tilted predictive $\pt$ has
  outpredicted the source predictive $\ps$ by a cumulative log-score
  margin exceeding $\log(1/\alpha)$ nats, an event that occurs with
  probability at most $\alpha$ under $H_0^{\rm pred}$.  This is what is
  directly confirmed.

  When the source predictive $\ps$ is well-calibrated, rejection further
  implies that the data are inconsistent with no shift and consistent with
  the posited tilt $\ws(y)$.  It does \emph{not} imply that $\ws(y)$ is
  the uniquely correct weight, since a different tilt in the same direction
  might also have been confirmed.  Rejection should therefore be read as
  \emph{the posited shift correction is supported by the data}, not as
  \emph{the posited correction is exactly right}.

  This distinction matters practically because if the posited shift is in the
  right direction but the wrong magnitude, the test may still reject,
  and the practitioner should not interpret rejection as a precise
  quantification of the shift.  For that, estimation methods
  \citep{LiptonZ2018icml,AlexandariAM2020icml} are the right tool.
  The confirmation test answers the binary question ``Is the correction plausible?'' rather than a quantitative one.
\end{remark}

\begin{proposition}[Growth rate and asymptotic power]
  \label{prop:power}
  Condition on the fixed training data $\Dtr$.  Suppose
  $(X_i,Y_i)\iid\nu(dx)\,\pt(dy\,|\,x,\Dtr)$, i.e.\ inputs are drawn
  i.i.d.\ from a marginal $\nu$ and, given $X_i=x$, outcomes follow the
  tilted predictive.  If
  \[
    \gamma\defeq
    \E_{X\sim\nu}\! \big[
      \KL \big( \pt(\cdot\,|\,X,\Dtr) \,\|\, \ps(\cdot\,|\,X,\Dtr) \big)
    \big]>0
  \]
  and $\E|\log e_i|<\infty$, then
  \[
    \frac{1}{t}\log M_t\to \gamma\quad\text{a.s.},
    \qquad
    \Prob_{\pt}(\tau^*<\infty)=1.
  \]
\end{proposition}

\emph{Proof.} See \cref{app:proof-prop-power}.

Proposition \ref{prop:power} has two complementary consequences.
First, the growth rate $\gamma$ equals the
average KL divergence between the tilted and source predictives.  Since
$\KL(\pt(\cdot\,|\,X,\Dtr)\|\ps(\cdot\,|\,X,\Dtr))\ge0$, the growth rate
is positive whenever $\pt(\cdot\,|\,X,\Dtr)\ne \ps(\cdot\,|\,X,\Dtr)$
on a set of positive $\nu$-measure.  In that case $\Prob_{\pt}(\tau^*<\infty)=1$,
so the test eventually rejects $H_0^{\rm pred}$ with probability one.
Second, $\gamma$ quantifies the speed.  A larger shift (higher KL) gives
faster rejection, with expected stopping time approximately
$\log(1/\alpha)/\gamma$.  
This tells the practitioner how many target observations to expect before confirmation, 
purely from the posited shift and the source predictive, with no target data required.
Together, Proposition~\ref{prop:anytime} (the test is safe) and
Proposition~\ref{prop:power} (the test is powerful, and quantifiably so
under correct specification) give a model-based frequentist
characterization of the confirmation procedure.

\begin{remark}[What happens under misspecification?]
  \label{rem:misspec}
  If the true conditional distribution is $q_t(\cdot\,|\,x)$ rather than $\pt$ or
  $\ps$, the expected log-growth is the cross-entropy gap
  \[
    \E_{q_t}\left[
      \log\pt(Y\,|\,X,\Dtr)-\log\ps(Y\,|\,X,\Dtr)
    \right].
  \]
  This quantity is positive exactly when the tilted predictive improves the
  expected log score relative to the source predictive.  However, exact
  Type I control under the data-level no-shift null $\qt=\qs$ is guaranteed
  only if $q_s(Y\,|\,X)=\ps(Y\,|\,X,\Dtr)$, or more generally if
  $\E_{q_s}[e_i\,|\, X_i,\Fil_{i-1}]\le1$.  Thus the present test is
  safe relative to a fixed predictive null; robustness to arbitrary
  predictive misspecification would require an additional calibration or
  conformalization layer.
\end{remark}

\subsection{Closed-Form E-Values for GP Sources}
\label{sec:closedform}

\begin{assumption}[GP source and Gaussian posited tilt]
  \label{ass:gp}
  $\ps(y\,|\, x,\Dtr)=\cN(y;\mu_s(x),\sigma_s^2(x))$.  The posited source
  and target label marginals used to form the weight are
  $p_s^{\rm lab}(y)=\cN(y;m_s,v_s^2)$ and
  $p_t^{\rm lab}(y)=\cN(y;m_t,v_t^2)$, with $v_s,v_t>0$.  For every tested input $x$,
  \[
    \Lambda(x)\defeq
    \frac{1}{\sigma_s^2(x)}+\frac{1}{v_t^2}-\frac{1}{v_s^2}>0.
  \]
\end{assumption}

\begin{remark}[Positive-precision condition in practice]
  \label{rem:posprec}
  The necessary and sufficient condition for the Gaussian normalization below
  is $\Lambda(x)>0$.  The simpler condition $\sigma_s^2(x)<v_t^2$ is neither
  necessary nor sufficient in general.  If $v_t^2\le v_s^2$, then
  $\Lambda(x)>0$ automatically for every $\sigma_s^2(x)>0$.  If
  $v_t^2>v_s^2$, then the condition becomes
  $\sigma_s^2(x)<v_s^2v_t^2/(v_t^2-v_s^2)$.  When $\Lambda(x)\le0$, the
  untruncated Gaussian-ratio weight does not define a finite normalizer at
  that input; the closed form is not merely unavailable.  One must either
  restrict testing to inputs satisfying $\Lambda(x)>0$, modify or truncate the
  posited weight so that $Z_w(x)<\infty$, or use a different predictive/weight
  pair for which the normalizer is finite.  Numerical quadrature is appropriate
  only after such finiteness has been guaranteed.  The e-process property is
  preserved exactly whenever the same finite $Z_w(x)$ is used in both $\pt$ and
  $e_i$.
\end{remark}

Under \cref{ass:gp}, the posited weight $\ws(y)=(v_s/v_t)\exp\bigl((y-m_s)^2/(2v_s^2)
-(y-m_t)^2/(2v_t^2)\bigr)$ is exponential-quadratic.

\begin{proposition}[Closed-form normalization]
  \label{prop:Zclosed}
  Under \cref{ass:gp}, with
  $\Lambda(x)\defeq 1/\sigma_s^2(x)+1/v_t^2-1/v_s^2$ and
  $\mu^*(x)\defeq\Lambda(x)^{-1}(\mu_s(x)/\sigma_s^2(x)+m_t/v_t^2-m_s/v_s^2)$,
  \[
    \Zx = \eta(x)    \exp\!\Bigl(
      \tfrac{(\mu^*(x))^2\Lambda(x)}{2}
      -\tfrac{\mu_s^2(x)}{2\sigma_s^2(x)}
      -\tfrac{m_t^2}{2v_t^2}
      +\tfrac{m_s^2}{2v_s^2}
    \Bigr),
  \]
  where $\eta(x) =  \frac{v_s}{v_t\,\sigma_s(x)\sqrt{\Lambda(x)}}$.
\end{proposition}

\emph{Proof.} See \cref{app:proof-prop-Zclosed}.

\begin{proposition}[Closed-form per-sample log e-value]
  \label{prop:logeval}
  Under \cref{ass:gp}, the tilted predictive is
  $\pt(\cdot\,|\, x,\Dtr)=\cN(\mu^*(x),\Lambda(x)^{-1})$, and
  \begin{align*}
    \log\ei
    &= \frac{(y_i-\mu_s(x_i))^2}{2\sigma_s^2(x_i)}
    -\frac{\Lambda(x_i)}{2}(y_i-\mu^*(x_i))^2 \\
    &\quad + \frac{1}{2}\log\!\bigl(\Lambda(x_i)\,\sigma_s^2(x_i)\bigr),
  \end{align*}
  computable from GP posterior outputs $(\mu_s(x_i),\sigma_s^2(x_i))$ and
  the shift parameters $(m_s,v_s^2,m_t,v_t^2)$ alone.
\end{proposition}

\emph{Proof.} See \cref{app:proof-prop-logeval}.

\begin{remark}[$p(x)$ cancellation]
  Under label shift, $\ws(y)$ does not depend on $x$, so $p_s(x)$ cancels
  from all expressions.  \cref{prop:logeval} inherits this because $\log\ei$
  depends on $x_i$ only through GP posterior outputs $\mu_s(x_i)$ and
  $\sigma_s^2(x_i)$, requiring no knowledge of the feature distribution.
\end{remark}

\subsection{Application to FDR Control and Robustness Sets}
\label{sec:extensions}

When $K$ streams are tested in parallel, applying the test $K$ times
at level $\alpha$ inflates the expected number of false confirmations to
$K\alpha$.  Bonferroni correction is too conservative (it targets
family-wise error).  
Instead, the e-BH procedure of \citet{WangR2022jrsssb} controls the \emph{false discovery rate} (FDR),
the expected fraction of false confirmations among all confirmed streams, at level $\alpha$.
For each compound $k$, compute the batch e-value
\[
  e^{(k)} \defeq M_{n_k}^{(k)}
  = \exp\!\Bigl(\sum_{i=1}^{n_k}
    \bigl(\NLPD_i^{s,(k)} - \NLPD_i^{w,(k)}\bigr)\Bigr),
\]
which is valid under $H_0^{\mathrm{pred},(k)}$ by \cref{prop:eprocess}.
Ordering $e^{(1)}\ge\cdots\ge e^{(K)}$ and setting
$k^* = \max\{k : e^{(k)} \ge K/(\alpha k)\}$, rejecting the top-$k^*$
controls $\mathrm{FDR}\le\alpha$ under independence across streams.  The
threshold $K/(\alpha k)$ relaxes progressively with rank, giving e-BH
more power than Bonferroni.  We recommend the batch e-value as the e-BH
input since it is order-free; using a sequentially stopped value
$M_{\tau^*}^{(k)}$ is valid \citep{RamdasA2023ss} but requires a
careful per-compound protocol.

If the practitioner is uncertain about the exact weight, any fixed
$\ws\in\cW_\beta$ from a robustness set (e.g.\ a $\chi^2$ ball) gives a
valid e-process by \cref{prop:eprocess}.  A mixture e-process
$M_t^{\rm mix}=\int M_t(w)\,d\Pi(w)$ over a prior $\Pi$ chosen before
testing is also valid; designing growth-rate-optimal safe tests for the
composite alternative connects to GRO e-values
\citep{GrunwaldP2024jrsssb} and is left to future work.

\section{Confirmation Test in Practice}
\label{sec:algorithm}

The method requires a strict two-stage separation. 
In the first stage, the practitioner fixes the weight $\ws(y)$ using any available
information \emph{before} the testing stream begins.  This could be an
off-the-shelf label-shift estimator such as BBSE
\citep{LiptonZ2018icml} or MLLS \citep{AlexandariAM2020icml} applied to
a \emph{separate} pilot dataset, domain expert knowledge, a regulatory
specification, or a bridging experiment.  In the second stage, the fixed
$\ws$ is handed to Algorithm~\ref{alg:confirmation}, which accumulates
sequential evidence for or against it on the incoming target stream.  The
validity guarantee (\cref{prop:anytime}) depends critically on this
separation because $\ws$ must not depend on the observations being tested, since
the e-value $e_i = \ws(y_i)/\Zx[x_i]$ is a valid e-value only when $\ws$
is committed to before $y_i$ is observed.  If $\ws$ was estimated on an
independent pilot batch, that batch must be disjoint from the testing
stream.

Algorithm~\ref{alg:confirmation} summarizes the complete procedure.  The
offline phase fixes the source predictive and the posited shift parameters; it
requires no additional estimation from the target testing stream.  The online
phase adds only constant-time arithmetic after the GP predictive
$(\mu_s(x_i),\sigma_s^2(x_i))$ has been evaluated.  The method computes
$\log e_i$ via \cref{prop:logeval} and increments the running log-wealth $S_t$.  By
\eqref{eq:nlpdgap}, $S_t=\log M_t$ is exactly the cumulative NLPD gap,
so practitioners already tracking that gap are implicitly running the test.

\begin{algorithm2e}[t]
\DontPrintSemicolon
\SetAlgoLined
\SetKwInOut{Input}{Inputs}
\SetKwInOut{Output}{Output}
\caption{Anytime-Valid Predictive Confirmation Test for a Posited Label-Shift Tilt}
\label{alg:confirmation}

\Input{%
  Source GP outputs $\mu_s(\cdot),\sigma_s^2(\cdot)$ from $\ps(y\,|\, x,\Dtr)$\;
  Posited shift parameters $(m_s,v_s^2,m_t,v_t^2)$\;
  Significance level $\alpha\in(0,1)$; optional monitoring horizon $T$}
\Output{Rejection time $\tau^*$, or ``not rejected'' after $T$ observations}

\BlankLine
\tcp{\textit{\footnotesize Offline stage. Commit to the source predictive and the posited tilt}}
Fix $(m_s,v_s^2,m_t,v_t^2)$ and set log-wealth $S_0\leftarrow 0$\;

\BlankLine
\tcp{\textit{\footnotesize Online stage. Process target observations sequentially}}
\For{$i = 1, 2, \ldots, T$}{
  Receive $(x_i, y_i)$\;
  Evaluate
  $\Lambda_i\leftarrow 1/\sigma_s^2(x_i)+1/v_t^2-1/v_s^2$ and
  $\mu_i^*\leftarrow\Lambda_i^{-1}\!\bigl(\mu_s(x_i)/\sigma_s^2(x_i)+m_t/v_t^2-m_s/v_s^2\bigr)$\;
  \If{$\Lambda_i\le0$}{\Return ``invalid Gaussian tilt at $x_i$; modify or truncate $w$ so that $Z_w(x_i)<\infty$''}
  Compute log e-value using \cref{prop:logeval}
  $\displaystyle\log e_i \leftarrow
    \frac{(y_i-\mu_s(x_i))^2}{2\sigma_s^2(x_i)}
    - \frac{\Lambda_i}{2}(y_i-\mu_i^*)^2
    + \frac{1}{2}\log\!\bigl(\Lambda_i\,\sigma_s^2(x_i)\bigr)$\;
  Update log-wealth, $S_i \leftarrow S_{i-1} + \log e_i$
    \tcp*{\footnotesize $S_i = \log M_i$ by \eqref{eq:nlpdgap}}
  \lIf{$S_i > \log(1/\alpha)$}{\Return $\tau^*\leftarrow i$
    \tcp*{\textit{\footnotesize tilted predictive confirmed (\cref{prop:anytime})}}}
}
\Return ``not rejected after $T$ observations''
\end{algorithm2e}

\if(0)
\paragraph{Computational cost.}
Each iteration requires $O(1)$ arithmetic given GP outputs
$(\mu_s(x_i),\sigma_s^2(x_i))$.  GP posterior evaluation is $O(n)$ per
prediction after $O(n^2)$ training; the confirmation test adds no
asymptotic overhead.  For the parallel FDR setting (\cref{sec:extensions}),
the loop runs independently per compound and e-BH post-processing costs
$O(K\log K)$.
\fi

\section{Experiments}
\label{sec:experiments}

We illustrate the method through end-to-end GP regression simulations that
follow the intended two-stage workflow.  We train a source GP on source data,
specify a label-shift correction, and accumulate closed-form e-values on
target observations.  This design tests the full pipeline, including GP
posterior uncertainty propagation into the e-values.

In each trial we generate a synthetic source training set
\[
  y_j^{\rm tr}=f(x_j^{\rm tr})+\epsilon_j,
\]
where $x_j^{\rm tr}\sim\mathrm{Unif}[-3,3]$ and $\epsilon_j\sim\cN(0,\sigma_\epsilon^2)$,
with $f(x)=\sin(2x)+0.5\cos(4x)$.
We fit a GP with an RBF kernel to
$\Dtr=\{(x_j^{\rm tr},y_j^{\rm tr})\}_{j=1}^{n_{\rm tr}}$.
The GP posterior predictive
\[
  \ps(y\mid x,\Dtr)=\cN(y;\mu_s(x),\sigma_s^2(x))
\]
serves as the source predictive.  Unless stated otherwise, $n_{\rm tr}=50$,
$\sigma_\epsilon^2=1$, RBF amplitude $1$ and lengthscale $0.8$,
$\alpha=0.05$.  The posited label marginals are
$p_s^{\rm lab}(y)=\cN(y;0,2)$ and $p_t^{\rm lab}(y)=\cN(y;m_t,3)$, and
shift magnitude is $\Delta=(m_t-m_s)/\sqrt{v_s^2}$.  Under $H_0^{\rm pred}$,
target labels are drawn from $\ps(\cdot\mid x,\Dtr)$; under the alternative,
from $\pt(\cdot\mid x,\Dtr)$.

\paragraph{Experiment 1. E-process trajectories.}
\cref{fig:exp1} shows median $\log M_t$ with IQR bands across 500 trials
for one null stream and for alternatives with $\Delta\in\{0.5,1.0,1.5\}$.
Under $H_0^{\rm pred}$, log-wealth fluctuates around zero and stays well
below the rejection boundary $\log(1/\alpha)\approx3$.  Under the alternative, log-wealth
grows approximately linearly and the empirical slope closely tracks the
conditional KL rate $\gamma$ predicted by \cref{prop:power}; the slope
increases with $\Delta$.  This confirms that the growth-rate
characterization holds in the end-to-end GP setting, where $\mu_s(x)$
and $\sigma_s^2(x)$ come from a fitted model rather than from a
theoretical oracle.

\begin{figure}[ht]
  \centering
  \includegraphics[width=\columnwidth]{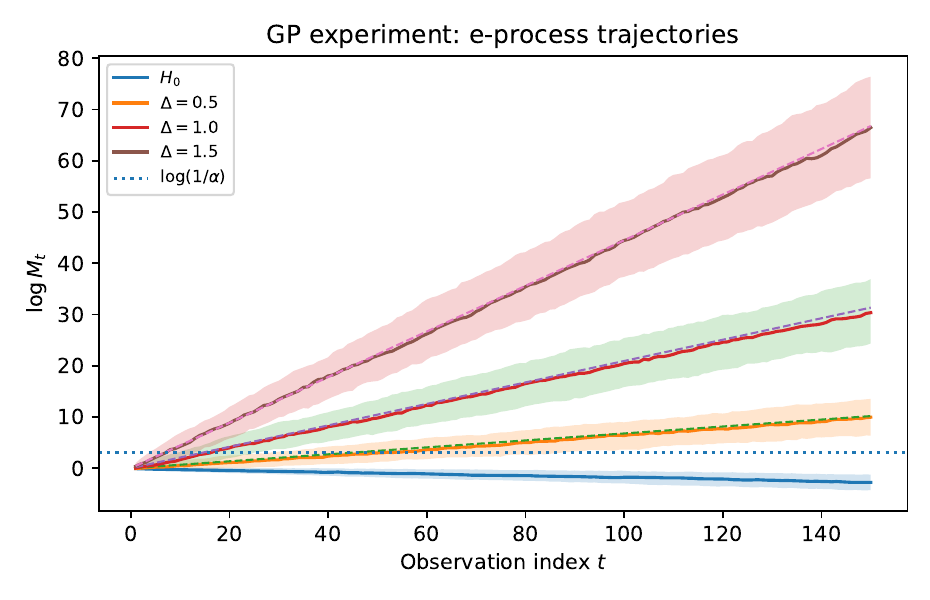}
  \caption{E-process trajectories under end-to-end GP regression
    ($n_{\rm tr}=50$, 500 trials).  Each trial trains a GP on source data
    and computes closed-form e-values from the GP posterior.  Dashed lines
    show the KL growth rate predicted by \cref{prop:power}; the dotted
    line is $\log(1/\alpha)$.}
  \label{fig:exp1}
\end{figure}

\paragraph{Experiment 2. Type I error under the predictive null.}
\cref{fig:exp2} plots the empirical false-alarm probability
$\widehat\alpha(T)=\Prob_{H_0^{\rm pred}}(\tau^*\le T)$ over 2000 trials
when target labels are drawn from the fitted source GP predictive.  The
curve stays strictly below $\alpha=0.05$ at every displayed horizon,
confirming the anytime-valid guarantee of \cref{prop:anytime} in the
end-to-end GP setting.  The overall false-alarm rate at $T=200$ is
$3.25\%$, demonstrating that the closed-form e-values are conservative in
this GP regression setup.

\begin{figure}[ht]
  \centering
  \includegraphics[width=\columnwidth]{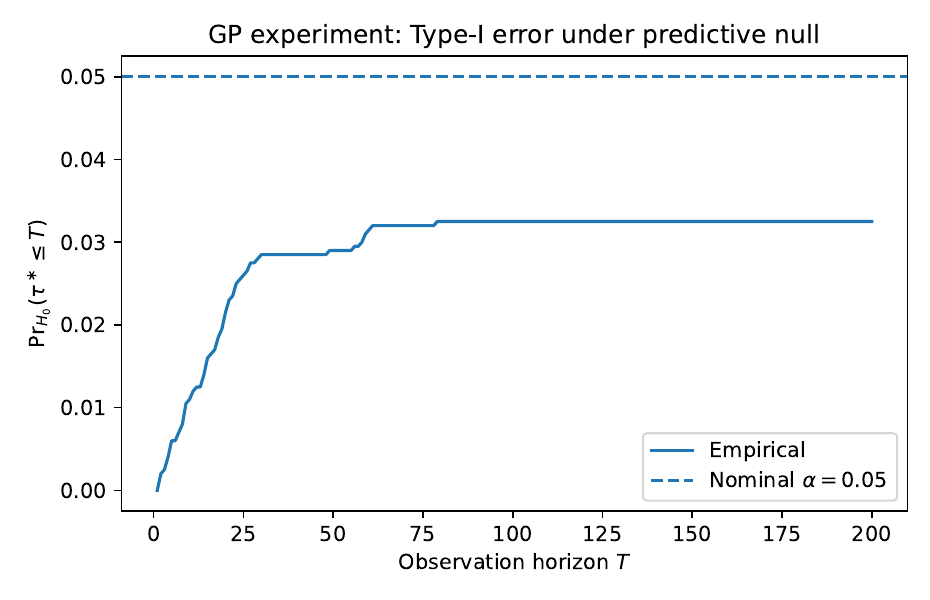}
  \caption{Type I error under $H_0^{\rm pred}$ ($n_{\rm tr}=50$, 2000
    trials).  The empirical false-alarm probability $\widehat\alpha(T)$
    stays below $\alpha=0.05$ (dashed) at every horizon, confirming the
    anytime-valid guarantee with a fitted GP source predictive.}
  \label{fig:exp2}
\end{figure}

\paragraph{Experiment 3. Power vs fixed-sample LRT.}
\cref{fig:exp3} compares the anytime-valid e-process with a fixed-sample
likelihood-ratio test (LRT) calibrated separately at each horizon under
$H_0^{\rm pred}$, at $\Delta=1.0$ (1500 trials).  The
fixed-sample LRT knows the sample size in advance and is optimally
calibrated at each horizon; the e-process does not.  Accordingly, the
LRT has a power advantage at large horizons.  At $n=50$, power is $0.993$
for the LRT and $0.963$ for the e-process.  However, the e-process
provides validity under \emph{arbitrary} stopping.  Its $\alpha$ guarantee
holds regardless of when the practitioner stops, with no recalibration.
The LRT, by contrast, is invalid if the practitioner stops early or
peeks at interim results.  This is the fundamental trade-off.  The
e-process sacrifices a small amount of power at any fixed horizon in
exchange for the freedom to stop at any time.

\begin{figure}[ht]
  \centering
  \includegraphics[width=\columnwidth]{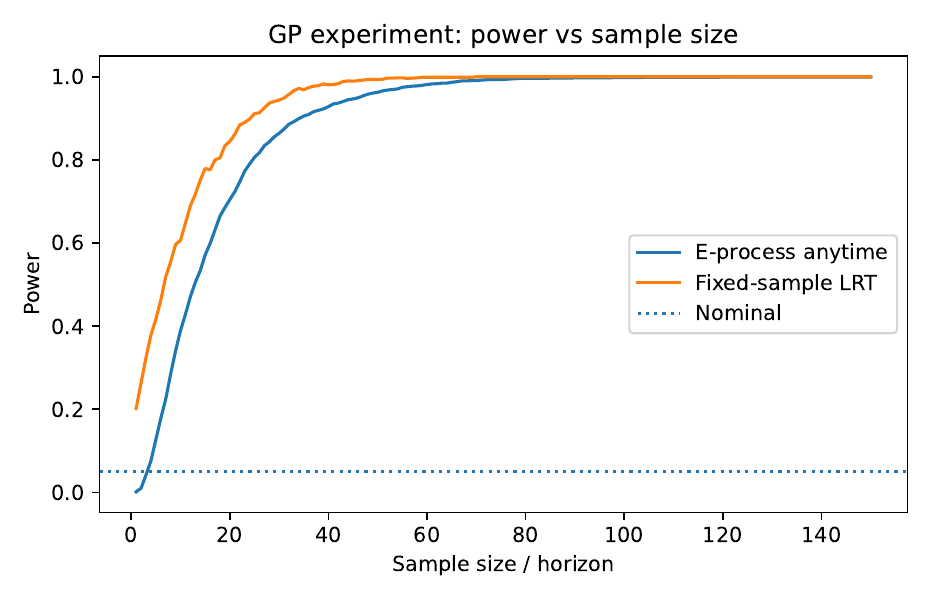}
  \caption{Power vs sample size at $\Delta=1.0$ ($n_{\rm tr}=50$,
    1500 trials).  The fixed-sample LRT is calibrated separately at each
    horizon and is valid only at the pre-specified $n$; the e-process is
    valid under any stopping time.  The power gap reflects the cost of
    anytime-validity.}
  \label{fig:exp3}
\end{figure}

\paragraph{Experiment 4. Reliable prior shift vs adaptive re-estimation.}
For each target batch size $n_t$, we compare three predictives on a
held-out target set.  The three predictives are (i)~the uncorrected source
GP, (ii)~the oracle corrected predictive using the true shift weight, and
(iii)~an adaptive correction that estimates $m_t$ from the $n_t$ target
observations and applies the plug-in tilt.  \cref{fig:exp4} shows that
the oracle correction gives substantially lower held-out NLPD than the
source predictive at all $n_t$, while re-estimation is competitive only
for larger batches.  At $n_t=10$, mean NLPD is $1.794$ (source), $1.586$
(oracle correction), and $1.652$ (re-estimated), confirming the
small-batch advantage of a reliable prior specification.  This motivates
the confirmation test.  When a trustworthy prior tilt is available, using
it beats noisy re-estimation, and the confirmation test provides formal
evidence that the tilt is consistent with the data.

\begin{figure}[ht]
  \centering
  \includegraphics[width=\columnwidth]{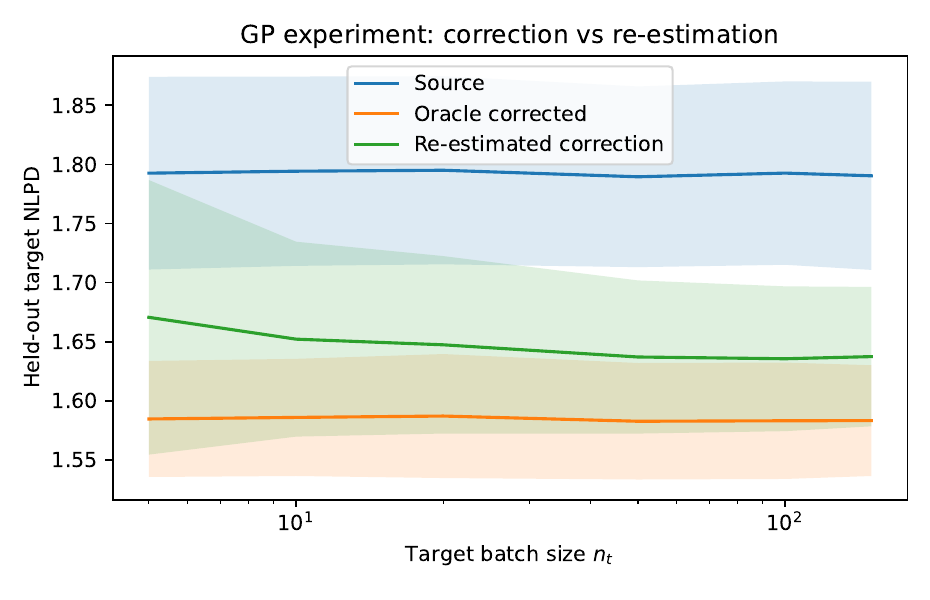}
  \caption{Oracle correction vs adaptive re-estimation under end-to-end
    GP regression.  Held-out target NLPD vs target batch size $n_t$.  A
    reliable prior tilt dominates at small $n_t$; re-estimation improves
    as $n_t$ grows but remains noisier in the small-batch regime.}
  \label{fig:exp4}
\end{figure}

\paragraph{Experiment 5. Source miscalibration stress test.}
The predictive-null guarantee requires observations to follow the source
predictive.  To stress-test this, we generate null labels from
$\cN(\mu_s(x),\,c\,\sigma_s^2(x))$ but compute e-values using the
original GP predictive $\cN(\mu_s(x),\sigma_s^2(x))$, simulating a GP
whose posterior variance is off by a factor $c$.  \cref{fig:exp5} shows
that at $c=1$ (correct calibration) Type I error is $3.2\%$, well below
$\alpha=0.05$.  Variance underestimation ($c>1$) is severely
anti-conservative.  Type I error rises to $36.45\%$ at $c=2$ and
$90.65\%$ at $c=4$.  Variance overestimation ($c<1$) is conservative.
This confirms that the predictive-null caveat is operationally
significant because a poorly calibrated GP posterior can substantially inflate
false-alarm rates, and practitioners should verify GP calibration before
deployment.

\begin{figure}[ht]
  \centering
  \includegraphics[width=\columnwidth]{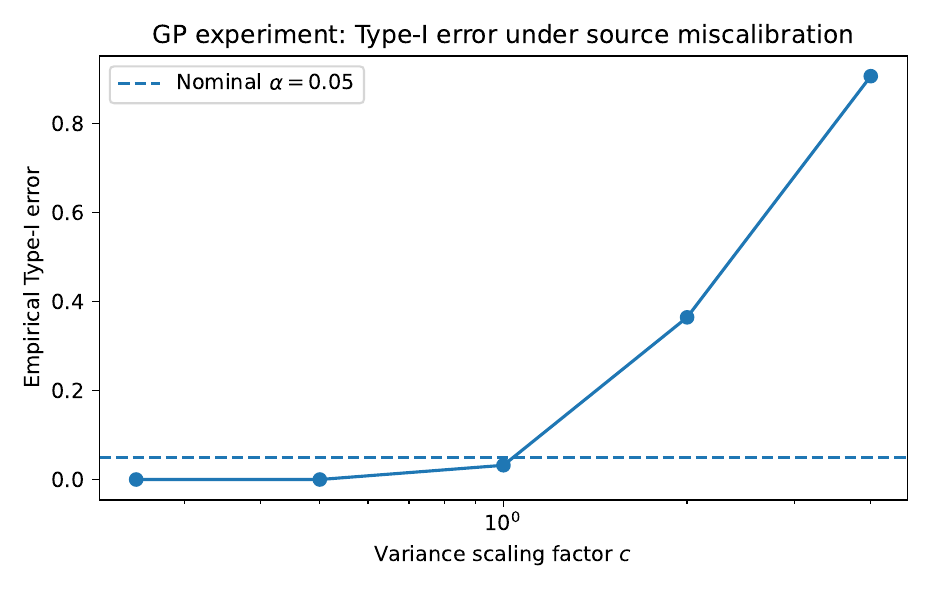}
  \caption{Type I error under GP posterior miscalibration.  Null labels
    are drawn from $\cN(\mu_s(x),c\sigma_s^2(x))$ but e-values use the
    original GP predictive ($c=1$).  Variance underestimation ($c>1$) inflates Type I error sharply;
    $c=1$ is calibrated, while $c<1$ is conservative.}
  \label{fig:exp5}
\end{figure}

Order-sensitivity and batch-e-value comparison experiments are omitted
for space; results are consistent with the theoretical analysis in
\cref{rem:misspec}.

\paragraph{Experiment 6. Misspecified tilt sensitivity.}
The previous experiments assume the posited shift weight is correctly
specified.  Here we ask what happens when the practitioner's prior is
wrong.  We generate target labels from the correctly tilted predictive
at $\Delta=1.0$, but compute e-values using a posited target
mean offset by $\delta\in\{-1.0,-0.5,0,+0.5,+1.0\}$ (in units of
$\sqrt{v_s^2}$, 1000 trials).  \cref{fig:exp6a} shows power curves and
\cref{fig:exp6b} shows mean $\log M_{50}$.

Two findings confirm the theoretical asymmetry.  First, moderate
misspecification reduces power, especially when the posited shift
undershoots the true shift, but the test remains informative and power
improves with batch size.  Second, and more importantly, tilt
misspecification cannot inflate Type I error under the predictive null
as long as the weight is fixed before testing and $Z_w(x_i)$ is finite.
\cref{prop:evalue,prop:eprocess} still give a valid e-process.  This
asymmetry is the key practical reassurance for practitioners who must
specify $w(y)$ from imperfect prior knowledge.  Validity is robust to the
posited tilt, while power degrades with the magnitude and direction of the
error.

\begin{figure}[ht]
  \centering
  \includegraphics[width=\columnwidth]{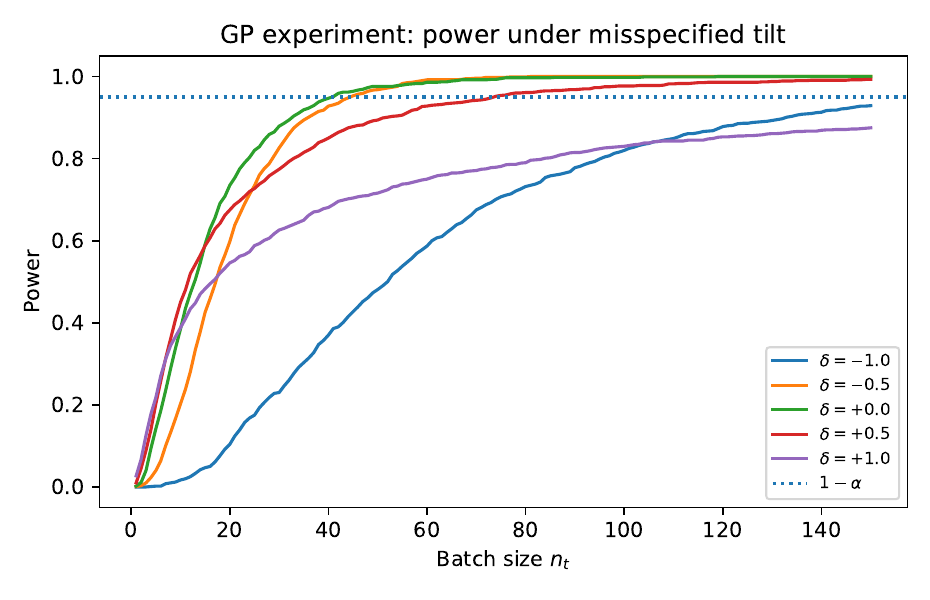}
  \caption{Power under misspecified tilt ($n_{\rm tr}=50$, 1000 trials,
    $\Delta=1.0$).  The true target is generated from the
    correctly tilted predictive; the test uses a posited mean offset by
    $\delta$.  Moderate misspecification reduces power but does not break
    the test.}
  \label{fig:exp6a}
\end{figure}

\begin{figure}[ht]
  \centering
  \includegraphics[width=\columnwidth]{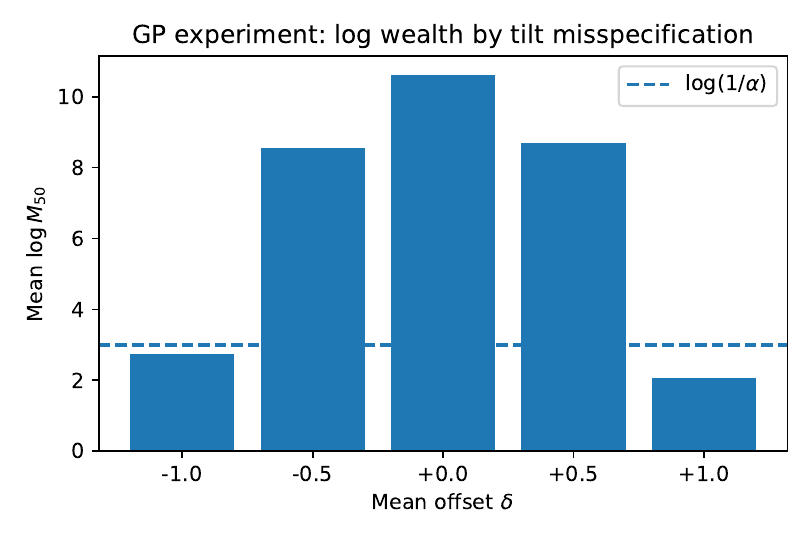}
  \caption{Mean $\log M_{50}$ under misspecified tilt.  The dashed line
    is the rejection threshold $\log(1/\alpha)$.  All offsets $\delta$
    yield positive mean log-wealth, although severe overstatement of the
    shift may fall below the rejection threshold at $n_t=50$.}
  \label{fig:exp6b}
\end{figure}

\section{Conclusion}
\label{sec:conclusion}

When labeled target outcomes are scarce and unlabeled target data alone
are insufficient for reliable shift estimation, re-estimating a label
shift from scratch is unreliable.  This paper provides a formal tool for the complementary
setting in which the practitioner has a domain-knowledge-based prior on the
shift direction and needs statistical evidence that it is consistent
with the incoming data.  We showed that the cumulative NLPD gap between
a fixed source predictive and a pre-specified tilted predictive is the
logarithm of a valid e-process under the source predictive null.
The normalized tilted/source likelihood ratios are conditional e-values,
their product is a nonnegative martingale, and Ville's inequality yields
an anytime-valid confirmation rule with no pre-specified sample size.
For GP sources with Gaussian label-shift tilts, e-values are available
in closed form.  Rejection confirms that the posited shift is
directionally supported by the data, not that it is exactly correct;
for quantitative shift estimation from target covariates,
unlabeled-target adaptation methods such as BBSE remain the appropriate
tool.

The main limitation is predictive-null validity.
Type I control holds when the source predictive is correctly calibrated,
not under arbitrary model misspecification.  Experiment~5 shows this is
operationally significant.  Variance underestimation by a factor of 2
inflates the false-alarm rate to $36.45\%$.  Practitioners should verify
GP calibration before deployment.  The candidate weight must also be
fixed before the testing stream begins.  Extensions to distribution-free
validity via conformalization, and to composite or misspecified-shift
families via mixture e-processes \citep{GrunwaldP2024jrsssb}, are
natural future directions.

\bibliographystyle{icml2026}
\bibliography{/users/seungjin/pub/bib/sjc}

\newpage
\onecolumn
\appendix

\section{Proofs}
\label{app:proofs}

\subsection{Proof of \cref{lem:lr}}
\label{app:proof-lem-lr}

\begin{proof}
For fixed $x$ with $0<Z_w(x)<\infty$, the tilted predictive satisfies
\[
  \Prob_w(dy\mid x,\Dtr)=\frac{w(y)}{Z_w(x)}\,\Prob_s(dy\mid x,\Dtr).
\]
Thus $\Prob_w(\cdot\mid x,\Dtr)$ is absolutely continuous with respect to
$\Prob_s(\cdot\mid x,\Dtr)$ and its Radon--Nikodym derivative is
\[
  \frac{d\Prob_w(\cdot\mid x,\Dtr)}{d\Prob_s(\cdot\mid x,\Dtr)}(y)
  =\frac{w(y)}{Z_w(x)}
  \qquad \Prob_s(\cdot\mid x,\Dtr)\text{-a.s.}
\]
In density notation this is
\[
  \frac{\pt(y\mid x,\Dtr)}{\ps(y\mid x,\Dtr)}=\frac{w(y)}{Z_w(x)}
\]
wherever the density ratio is defined.  Finally,
\[
  \E_{Y\sim\ps(\cdot\mid x,\Dtr)}\left[\frac{w(Y)}{Z_w(x)}\right]
  =\frac{1}{Z_w(x)}\int w(y)\ps(y\mid x,\Dtr)dy=1,
\]
by the definition of $Z_w(x)$.
\end{proof}

\subsection{Proof of \cref{prop:evalue}}
\label{app:proof-prop-evalue}

\begin{proof}
Condition on $(\Fil_{i-1},X_i)$.  Under $H_0^{\rm pred}$, the
conditional distribution of $Y_i$ is $\ps(\cdot\mid X_i,\Dtr)$.  Therefore, by
\cref{lem:lr},
\[
  \E[e_i\mid\Fil_{i-1},X_i]
  =\E_{Y\sim\ps(\cdot\mid X_i,\Dtr)}
    \left[\frac{w(Y)}{Z_w(X_i)}\right]
  =1.
\]
Since $e_i\ge0$, this proves that $e_i$ is a conditional e-value given the
realized input $X_i$.
\end{proof}

\subsection{Proof of \cref{prop:eprocess}}
\label{app:proof-thm-eprocess}

\begin{proof}
Non-negativity is immediate because each $e_i\ge0$.  Moreover, by
\cref{prop:evalue},
\[
  \E[e_t\mid\Fil_{t-1},X_t]=1.
\]
Thus $e_t$ is conditionally integrable, and by the tower property,
\[
  \E[e_t\mid\Fil_{t-1}]
  =\E\!\left[\E[e_t\mid\Fil_{t-1},X_t]\mid\Fil_{t-1}\right]
  =\E[1\mid\Fil_{t-1}]=1.
\]
Since $M_{t-1}$ is $\Fil_{t-1}$-measurable,
\[
  \E[M_t\mid\Fil_{t-1}]
  =\E[M_{t-1}e_t\mid\Fil_{t-1}]
  =M_{t-1}\E[e_t\mid\Fil_{t-1}]
  =M_{t-1}.
\]
Therefore $(M_t)_{t\ge0}$ is a nonnegative martingale with $M_0=1$, and
hence an e-process under $H_0^{\rm pred}$.
\end{proof}

\subsection{Proof of \cref{prop:anytime}}
\label{app:proof-thm-anytime}

\begin{proof}
By \cref{prop:eprocess}, $(M_t)_{t\ge0}$ is a nonnegative martingale under
$H_0^{\rm pred}$ with $M_0=1$.  
Ville's inequality gives
\[
  \Prob_{H_0^{\rm pred}}\!\left(\sup_{t\ge0}M_t>\frac{1}{\alpha}\right)
  \le \alpha.
\]
The event $\{\tau^*<\infty\}$ is contained in the event
$\{\sup_{t\ge0}M_t>1/\alpha\}$, so
\[
  \Prob_{H_0^{\rm pred}}(\tau^*<\infty)\le\alpha.
\]
\end{proof}

\subsection{Proof of \cref{prop:power}}
\label{app:proof-prop-power}

\begin{proof}
Condition on the fixed training data $\Dtr$.  Under the assumption
$(X_i,Y_i)\iid \nu(dx)\pt(dy\mid x,\Dtr)$, the variables $\log e_i$ are
i.i.d. and integrable.  For any fixed $x$,
\begin{align*}
  \E_{\pt(\cdot\mid x,\Dtr)}[\log e_i]
  &=\int \pt(y\mid x,\Dtr)
      \log\frac{\pt(y\mid x,\Dtr)}{\ps(y\mid x,\Dtr)}dy \\
  &=\KL\{\pt(\cdot\mid x,\Dtr)\|\ps(\cdot\mid x,\Dtr)\}.
\end{align*}
Taking expectation over $X\sim\nu$ gives
\[
  \E[\log e_i]=
  \E_{X\sim\nu}\left[
    \KL\{\pt(\cdot\mid X,\Dtr)\|\ps(\cdot\mid X,\Dtr)\}
  \right]=\gamma.
\]
The strong law of large numbers yields
\[
  \frac{1}{t}\log M_t
  =\frac{1}{t}\sum_{i=1}^t\log e_i
  \to \gamma
  \qquad \text{a.s.}
\]
If $\gamma>0$, then $\log M_t\to\infty$ almost surely.  Hence the fixed
threshold $\log(1/\alpha)$ is eventually crossed with probability one, so
$\Prob_{\pt}(\tau^*<\infty)=1$.
\end{proof}

\subsection{Proof of \cref{prop:Zclosed}}
\label{app:proof-prop-Zclosed}

\begin{proof}
We compute $Z_w(x)=\int w(y)\ps(y\mid x,\Dtr)dy$ by substituting
\cref{ass:gp}.
\[
  Z_w(x)=\int_{-\infty}^{\infty}
        \frac{v_s}{v_t}
        \exp\!\left(\frac{(y-m_s)^2}{2v_s^2}-\frac{(y-m_t)^2}{2v_t^2}\right)
        \frac{1}{\sqrt{2\pi}\sigma_s(x)}
        \exp\!\left(-\frac{(y-\mu_s(x))^2}{2\sigma_s^2(x)}\right)dy.
\]
The exponent can be written as
\[
  -\frac{\Lambda(x)}{2}y^2+y\left\{\frac{\mu_s(x)}{\sigma_s^2(x)}
  +\frac{m_t}{v_t^2}-\frac{m_s}{v_s^2}\right\}+C_0,
\]
where
\[
  C_0=-\frac{\mu_s^2(x)}{2\sigma_s^2(x)}-\frac{m_t^2}{2v_t^2}
      +\frac{m_s^2}{2v_s^2}.
\]
By definition,
\[
  \mu^*(x)=\Lambda(x)^{-1}\left\{\frac{\mu_s(x)}{\sigma_s^2(x)}
  +\frac{m_t}{v_t^2}-\frac{m_s}{v_s^2}\right\},
\]
so the exponent is
\[
  -\frac{\Lambda(x)}{2}\{y-\mu^*(x)\}^2
  +C_0+\frac{\Lambda(x)(\mu^*(x))^2}{2}.
\]
Since $\Lambda(x)>0$ by \cref{ass:gp},
\[
  \int_{-\infty}^{\infty}
  \exp\!\left[-\frac{\Lambda(x)}{2}\{y-\mu^*(x)\}^2\right]dy
  =\sqrt{\frac{2\pi}{\Lambda(x)}}.
\]
Combining the prefactors gives
\[
  Z_w(x)=\frac{v_s}{v_t\sigma_s(x)\sqrt{\Lambda(x)}}
  \exp\!\left(
      \frac{\Lambda(x)(\mu^*(x))^2}{2}
      -\frac{\mu_s^2(x)}{2\sigma_s^2(x)}
      -\frac{m_t^2}{2v_t^2}
      +\frac{m_s^2}{2v_s^2}
  \right),
\]
which is the stated expression.
\end{proof}

\subsection{Proof of \cref{prop:logeval}}
\label{app:proof-prop-logeval}

\begin{proof}
The completing-the-square calculation in the proof of \cref{prop:Zclosed}
shows that the unnormalized tilted density is proportional to
\[
  \exp\!\left[-\frac{\Lambda(x)}{2}\{y-\mu^*(x)\}^2\right].
\]
After normalization, therefore,
\[
  \pt(\cdot\mid x,\Dtr)=\cN(\mu^*(x),\Lambda(x)^{-1}).
\]
At $(x_i,y_i)$,
\begin{align*}
  \log\pt(y_i\mid x_i,\Dtr)
  &=-\frac{\Lambda(x_i)}{2}\{y_i-\mu^*(x_i)\}^2
    -\frac{1}{2}\log\{2\pi\Lambda(x_i)^{-1}\},\\
  \log\ps(y_i\mid x_i,\Dtr)
  &=-\frac{\{y_i-\mu_s(x_i)\}^2}{2\sigma_s^2(x_i)}
    -\frac{1}{2}\log\{2\pi\sigma_s^2(x_i)\}.
\end{align*}
Subtracting the second display from the first yields
\[
  \log e_i
  =\frac{\{y_i-\mu_s(x_i)\}^2}{2\sigma_s^2(x_i)}
  -\frac{\Lambda(x_i)}{2}\{y_i-\mu^*(x_i)\}^2
  +\frac{1}{2}\log\{\Lambda(x_i)\sigma_s^2(x_i)\},
\]
which proves the stated formula.
\end{proof}

\end{document}